\documentclass[sigconf]{acmart}

\usepackage{booktabs} 
\usepackage{xspace} 
\usepackage{bbm}

\newtheorem{theorem}{Theorem}[section]
\newtheorem{proposition}[theorem]{Proposition}
\newtheorem{remark}{Remark}[section]

\newtheorem{definition}{Definition}[section]

\usepackage{algorithm}
\usepackage{algorithmic}

\newcommand{\E}{\mathop{\mathbb E}}
\newcommand{\Tr}{\operatorname{Tr}}

\newcommand{\N}{\mathcal{N}}

\def\x{{\mathbf x}}

\DeclareMathOperator*{\argmin}{arg \, min}

\renewcommand{\t}[1]{\mathrm{T}#1}
\renewcommand{\d}[1]{\;\mathrm{d}#1}
\renewcommand{\t}[1]{\mathrm{T}#1}

\setcopyright{rightsretained}

\acmDOI{10.1145/nnnnnnn.nnnnnnn}

\acmISBN{978-x-xxxx-xxxx-x/YY/MM}

\acmConference[]{A.I Square Working Paper}{March 2019}{France}
\acmYear{2019}
\copyrightyear{2019}

\acmPrice{15.00}

\begin{document}
\title{BCMA-ES: A Bayesian approach to CMA-ES}


%
%
%

\author{Eric Benhamou}
\affiliation{%
  \institution{LAMSADE and A.I Square Connect}
  \city{Paris Dauphine} 
  \state{France} 
  \postcode{92200}
}
\email{eric.benhamou@aisquareconnect.com}

\author{David Saltiel}
\affiliation{%
  \institution{LISIC and A.I Square Connect}
  \streetaddress{Maison de la Recherche Blaise Pascal,  50 rue Ferdinand Buisson BP 719}
  \city{ULCO-LISIC} 
  \state{France} 
  \postcode{62 228}
}
\email{david.saltiel@aisquareconnect.com}

\author{Sebastien Verel}
\affiliation{%
 \institution{ULCO-LISIC}
  \streetaddress{Maison de la Recherche Blaise Pascal,  50 rue Ferdinand Buisson BP 719}
  \state{France} 
  \postcode{62 228}
}
\email{verel@univ-littoral.fr}

\author{Fabien Teytaud}
\affiliation{%
  \institution{ULCO-LISIC}
  \streetaddress{Maison de la Recherche Blaise Pascal,  50 rue Ferdinand Buisson BP 719}
  \state{France} 
  \postcode{62 228}
}
\email{fabien.teytaud@univ-littoral.fr}

\renewcommand{\shortauthors}{Benhamou, Saltiel, Verel, Teytaud}

\begin{abstract}
This paper introduces a novel theoretically sound approach for the celebrated CMA-ES algorithm. Assuming the parameters of the multi variate normal distribution for the minimum follow a conjugate prior distribution, we derive their optimal update at each iteration step. Not only provides this Bayesian framework a justification for the update of the CMA-ES algorithm but it also gives two new versions of CMA-ES either assuming normal-Wishart or normal-Inverse Wishart priors, depending whether we parametrize the likelihood by its covariance or precision matrix. We support our theoretical findings by numerical experiments that show fast convergence of these modified versions of CMA-ES.
\end{abstract}

%
%



\maketitle

\section{Introduction}
The covariance matrix adaptation evolution strategy (CMA-ES) ~\cite{HansenOstermeier_2001} is arguably one of the most powerful real-valued derivative-free optimization algorithms, finding
many applications in machine learning. It is a state-of-the-art optimizer for continuous black-box functions as shown by the various benchmarks of the \href{http://coco.gforge.inria.fr/}{COmparing Continuous Optimisers} INRIA platform for ill-posed functions. It has led to a large number of papers and articles and we refer the interested reader to~\cite{HansenOstermeier_2001,Auger_2004,Igel_2007,Auger_2009,Hansen_2011,Auger_2012,Hansen_2014,Auger_2015,Auger_2016,Ollivier_2017} and~\cite{Hansen_2018} to cite a few.

It has has been successfully applied in many unbiased performance comparisons and numerous real-world applications. In particular, in machine learning, it has been used for direct policy search in reinforcement learning and hyper-parameter tuning in supervised learning (\cite{Igel_2009a,Igel_2009b,Igel_2010}), and references therein, as well as hyperparameter optimization of deep neural networks ~\cite{Loshchilov_2016}

In a nutshell, the ($\mu$ / $\lambda$) CMA-ES is an iterative black box optimization algorithm, that, in each of its iterations, samples $\lambda$ candidate solutions from a multivariate normal distribution, evaluates these solutions (sequentially or in parallel) retains $\mu$ candidates and adjusts the sampling distribution used for the next iteration to give higher probability to good samples. Each iteration can be individually seen as taking an initial guess or \emph{prior} for the multi variate parameters, namely the mean and the covariance, and after making an experiment by evaluating these sample points with the fit function updating the initial parameters accordingly. 

Historically, the CMA-ES has been developed heuristically, mainly by conducting experimental research and validating intuitions empirically. Research was done without much focus on theoretical foundations because of the apparent complexity of this algorithm. It was only recently that~\cite{Akimoto_2010,Glasmachers_2010} and~\cite{Ollivier_2017} made a breakthrough and provided a theoretical justification of CMA-ES updates thanks to information geometry. They proved that CMA-ES was performing a natural gradient descent in the Fisher information metric. These works provided nice explanation for the reasons of the performance of the CMA-ES because of strong invariance properties, good search directions, etc

There is however another way of explanation that has been so far ignored and could also bring nice insights about CMA-ES. It is Bayesian statistics theory. At the light of Bayesian statistics, CMA-ES can be seen as an iterative prior posterior update. But there is some real complexity due to tricky updates that may explain why this has always been ignored. First of all, in a regular Bayesian approach, all sample points are taken. This is not the case in the ($\mu$/$\lambda$) CMA-ES as out of the $\lambda$ generated paths, only the $\mu$ best are selected. The updating weights are also constant which is not consistent with Bayesian updates. But more importantly, the covariance matrix update is the core of the problem. It appeals important remarks. The update is done according to a weighted combination of a rank one matrix referred to $p_{C} p_{C}^T$ with parameter $c_1$ and a rank $min(\mu,n)$ matrix with parameter $c_{\mu}$, whose details are given for instance in~\cite{Hansen_2016}. The two updates for the covariance matrix makes the Bayesian update interpretation challenging as these updates are done according to two paths: the isotropic and anisotropic evolution path. All this may explain why a Bayesian approach for interpreting and revisiting the CMA-ES algorithm have seemed a daunting task and not tackled before.

This is precisely the objective of this paper. Section \ref{sec:framework} recalls various Bayesian concepts of updates for prior and posterior to highlight the analogy of an iterative Bayesian update. Section \ref{sec:algorithms} presents in greater details the Bayesian approach of CMA-ES, with the corresponding family of derived algorithms, emphasizing the various design choices that can conduct to multiple algorithms. Section \ref{sec:experiments} provides numerical experiments and shows that
Bayesian adapted CMA-ES algorithms perform well on convex and non convex functions. We finally conclude about some possible extensions and further experiments.

However, the analogy with a successive Bayesian prior posterior update has been so far missing in the landscape of CMA-ES for multiple reasons. First of all, from a cultural point of view, the evolutionary and Bayesian community have always been quite different and not overlapping. Secondly, the CMA-ES was never formulated in terms of a prior and posterior update making its connection with Bayesian world non obvious. Thirdly, when looking in details at the parameters updates, the weighted combination between the global and local search makes the interpretation of a Bayesian posterior update non trivial. We will explain in this paper that the global search needs to be addressed with a special dilatation techniques that is not common in Bayesian wold.

\section{Framework}\label{sec:framework}
CMA-ES computes at each step an update of the mean and covariance of the distribution of the minimum. From a very general point of view this can be interpreted as a prior posterior update in Bayesian statistics.

\subsection{Bayesian vs Frequentist probability theory}
The justification of the Bayesian approach is discussed in \cite{Robert_2007}. In Bayesian probability theory, we assume a distribution on unknown parameters of a statistical
model that can be characterized as a probabilization of uncertainty. This procedure leads to an axiomatic reduction from the notion of unknown to the notion of randomness but with probability. We do not know the value of the parameters for sure but we know specific values that these parameters can take with higher probabilities. This creates a prior distribution that is updated as we make some experiments as shown in ~\cite{Gelman_2004,Marin_2007,Robert_2007}.  In the Bayesian view, a probability is assigned to a hypothesis, whereas under frequentist inference, a hypothesis is typically tested without being assigned a probability. There are even some nice theoretical justification for it as presented in~\cite{Jordan_2010}.

\begin{definition} (Infinite exchangeability).
We say that $(x_1, x_2, . . .)$ is an infinitely exchangeable sequence of
random variables if, for any n, the joint probability $p(x_1, x_2, . . . , x_n)$ is invariant to permutation of the
indices. That is, for any permutation $\pi$,
$$
p(x_1, x_2, . . . , x_n) = p(x_{\pi1}, x_{\pi2}, . . . , x_{\pi n})
$$
\end{definition}

Equipped with this definition, the De Finetti's theorem as provided below states that exchangeable observations are conditionally independent relative to some latent variable.

\begin{theorem}
\label{theorem_finetti}
(De Finetti, 1930s). A sequence of random variables $(x_1, x_2, . . .)$ is infinitely exchangeable iff,
for all n,
$$
p(x_1, x_2, . . . , x_n) = \int \prod_{i=1}^n p(x_i|\theta) P(d\theta),
$$
for some measure P on $\theta$.
\end{theorem}

This representation theorem \ref{theorem_finetti} justifies the use of priors on parameters since for exchangeable data, there must exist a parameter $\theta$, a likelihood $p(x|\theta)$ and a distribution $\pi$ on $\theta$. A proof of De Finetti theorem is for instance given in~\cite{Schervish_1996} (section 1.5).

\begin{remark}
The De Finetti is trivially satisfied in case of i.i.d. sampling as the sequence is clearly exchangeable and that the joint probability is clearly given by the product of all the marginal distributions. However, the De Finetti goes far beyond as it proves that the infinite exchangeability is enough to prove that the joint distribution is the product of some marginal distribution for a given parameter $\theta$. The sequence may not be independent neither identically distributed, which is a much stronger result!
\end{remark}

\subsection{Conjugate priors}
In Bayesian statistical inference, the probability distribution that expresses one's (subjective) beliefs about the distribution parameters before any evidence is taken into account is called \emph{the prior} probability distribution, often simply called the prior. In CMA-ES, it is the distribution of the mean and covariance. We can then update our prior distribution with the data using Bayes' theorem to obtain a posterior distribution. The \emph{posterior} distribution is a probability distribution that represents your updated beliefs about the parameters after having seen the data. The Bayes' theorem tells us \emph{the fundamental rule} of Bayesian statistics, that is
$$
\text{Posterior} \propto \text{Prior} \times \text{Likelihood}
$$

\noindent The proportional sign indicates that one should compute the distribution up to a  renormalization constant that enforces the distribution sums to one. This rule is simply a direct consequence of Baye's theorem. Mathematically, let us say that for a random variable $X$, its distribution $p$ depends on a parameter $\theta$ that can be multi-dimensional. To emphasize the dependency of the distribution on the parameters, let us write this distribution as $p(x|\theta)$ and let us assume we have access to a prior distribution $\pi(\theta)$. Then the joint distribution of $(\theta,x)$ writes simply as 
$$
\phi(\theta,x) = p(x | \theta) \pi(\theta)
$$
The marginal distribution of $x$ is trivially given by marginalizing the joint distribution by $\theta$ as follows: 
$$
m(x) = \int \phi(\theta,x)  d\theta =  \int p(x | \theta) \pi(\theta) d\theta
$$
The posterior of $\theta$  is obtained by Bayes's formula as 
$$
\pi(\theta | x ) = \frac{  p(x | \theta) \pi(\theta) }{\int p(x | \theta) \pi(\theta) d\theta} \propto  p(x | \theta) \pi(\theta)
$$

Computing a posterior is tricky and does not bring much value in general. A key concept in Bayesian statistics is conjugate priors that makes the computation really easy and is described at length below.

\begin{definition}\label{sec:conjprior}
A prior distribution $\pi(\theta)$ is said to be a conjugate prior if the posterior distribution 
\begin{align}
\pi(\theta | x ) \propto   p(x | \theta) \pi(\theta) \label{eq:ExponentialFamilyPosterior}
\end{align}
remains in the same distribution family as the prior.
\end{definition}

At this stage, it is relevant to introduce exponential family distributions as this higher level of abstraction that encompasses the multi variate normal trivially solves the issue of founding conjugate priors. This will be very helpful for inferring conjugate priors for the multi variate Gaussian used in CMA-ES.

\begin{definition}
A distribution is said to belong to the exponential family if it can be written (in its canonical form) as:
\begin{equation}
p(\x | \eta)=h(\x)\exp(\eta \cdot T(\x)-A(\eta )),
\label{eq:ExponentialFamilyNaturalForm}
\end{equation}
where  $\eta$ is the natural parameter, $T(\x)$ is the sufficient statistic, $A(\eta)$ is log-partition function and  $h(\x)$ is the base measure. $\eta$ and  $T(\x)$ may be vector-valued. Here $a \cdot b$ denotes the inner product of $a$ and $b$.

\noindent The log-partition function is defined by the integral
\begin{align}
A(\eta)\triangleq \log \int_{\mathcal{X}}{h(\x)\exp({\eta \cdot T(\x)}) \d x}.
\end{align}
Also, $\eta \in \Omega=\{\eta \in \mathbb{R}^m |A(\theta)< +\infty\}$ where $\Omega$ is the natural parameter space. Moreover, $\Omega$ is a convex set and $A(\cdot)$ is a convex function on $\Omega$.
\end{definition}

\begin{remark} \label{ex:exponentialfamily} 
Not surprisingly, the normal distribution $\N(\x;\mu,\Sigma)$ with mean $\mu\in\mathbb{R}^d$ and covariance matrix $\Sigma$ belongs to the exponential family but with a different parametrisation. Its exponential family form is given by:
\begin{subequations}
\label{eq:normalExpForm}
\begin{align}
\eta(\mu,\Sigma)&=\begin{bmatrix} \Sigma^{-1}\mu \\  \mathrm{vec}(\Sigma^{-1}) \end{bmatrix}, \qquad T(\x) =\begin{bmatrix} \x\\\mathrm{vec}(-\frac{1}{2} \x\x^\t) \end{bmatrix} \label{eq:Natparams},\\
h(\x)&=(2\pi)^{-\frac{d}{2}}, \qquad  A(\eta(\mu,\Sigma)) =\frac{1}{2}\mu^\t \Sigma^{-1}\mu+\frac{1}{2}\log|\Sigma|.
\end{align} 
\end{subequations}
where in equations  \eqref{eq:Natparams}, the notation $\mathrm{vec}(\cdot)$ means we have vectorized the matrix, stacking each column on top of each other and hence can equivalently write for $a$ and $b$, two matrices, the trace result $\Tr(a^\t{}b)$ as the scalar product of their vectorization $\mathrm{vec}(a)\cdot \mathrm{vec}(b)$. We can remark the canonical parameters are very different from traditional (also called moment) parameters. We can notice that changing slightly the sufficient statistic $T(x)$ leads to change the corresponding canonical parameters $\eta$.
\end{remark}

\noindent For an exponential family distribution, it is particularly easy to form conjugate prior. 
\begin{proposition}
If the observations have a density of the exponential family form $p(x |\theta,\lambda) = h(x) \exp\Big(\eta(\theta,\lambda)^T T(x)-n A(\eta(\theta,\lambda))\Big)$, with $\lambda$ a set of hyper-parameters, then the prior with likelihood defined by $\pi(\theta) \propto \exp\left( \mu_1\cdot \eta(\theta,\lambda) - \mu_0 A(\eta(\theta,\lambda)) \right)$ with $\mu\triangleq(\mu_0, \mu_1)$ is a conjugate prior.
\end{proposition}

The proof is given in appendix subsection \ref{proof:conjugate_priors}. As we can vary the parameterisation of the likelihood, we can obtain multiple conjugate priors. Because of the conjugacy, if the initial parameters of the multi variate Gaussian follows the prior, the posterior is the true distribution given the information $\mathcal{X}$ and stay in the same family making the update of the parameters really easy. Said differently, with conjugate prior, we make the optimal update.  And it is enlightening to see that as we get some information about the likelihood, our posterior distribution becomes more peak as shown in figure\ref{fig:conjsin}.

\begin{figure}[ht]
\centering
\includegraphics[width=.35\textwidth]{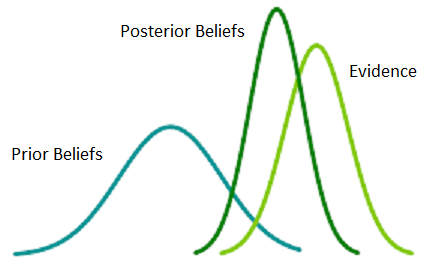}			
\caption{As we get more and more information using the likelihood, the posterior becomes more peak.}
\label{fig:conjsin}	
\end{figure}

\subsection{Optimal updates for NIW}
The two natural conjugate priors for the Multi variate normal that updates both the mean and the covariance are the normal-inverse-Wishart if we want to update the mean and covariance of the Multi variate normal or the normal-Wishart if we are interested in updating the mean and the precision matrix (which is the inverse of the covariance matrix). In this paper, we will stick to the normal-inverse-Wishart to keep things simple. The Normal-inverse-Wishart distribution is parametrized by $\boldsymbol\mu_0,\lambda,\boldsymbol\Psi,\nu$ and its distribution is given by
\[
f(\boldsymbol\mu,\boldsymbol\Sigma|\boldsymbol\mu_0,\lambda,\boldsymbol\Psi,\nu) = \mathcal{N}\left(\boldsymbol\mu\Big|\boldsymbol\mu_0,\frac{1}{\lambda}\boldsymbol\Sigma\right) \mathcal{W}^{-1}(\boldsymbol\Sigma|\boldsymbol\Psi,\nu)
\]
where $ \mathcal{W}^{-1}$ denotes the inverse Wishart distribution. The key theoretical guarantee of the BCMA-ES is to update the mean and covariance of our CMA-ES optimally as follows. 

\begin{proposition}\label{Bayesian_update}
If our sampling density follows a $d$ dimensional multivariate normal distribution $\sim \mathcal{N}_d(\boldsymbol\mu,\boldsymbol\Sigma)$ with unknown mean $\mu$ and covariance $\Sigma$ and if its parameters are distributed according to a Normal-Inverse-Wishart $
(\mu, \Sigma) \sim \mathrm{NIW}(\mu_0,\kappa_0,v_0,\psi)$ and if we observe $\mathcal{X}=(x_1,..,x_n)$ samples, then the posterior is also a Normal-Inverse-Wishart with different parameters $\mathrm{NIW}(\mu_0^\star,\kappa_0^\star,v_0^\star,\psi^\star)$ given by
\begin{equation}\label{eq:Bayesian_update}
\begin{split}
\mu_0^\star &= \frac{\kappa_0 \mu_0+n\overline{x}}{\kappa_0 +n}, \\
\kappa_0^\star &=\kappa_0+n, \\
v_0^\star &= v_0+n \\
\psi^\star &= \psi + \sum_{i=1}^n \left( x_i-\overline{x}\right) \left( x_i-\overline{x}\right)^T+\frac{\kappa_0 n}{\kappa_0 +n}\left(\overline{x}-\mu_0\right) \left(\overline{x}-\mu_0\right)^T
\end{split}
\end{equation}
with $\overline{x}$ the sample mean. 
\end{proposition}

\begin{remark}
This proposition is the cornerstone of the BCMA-ES. It provides the theoretical guarantee that the updates of the parameters in the algorithm are accurate and optimal under the assumption of the prior. In particular, this implies that any other formula for the update of the mean and variance and in particular the ones used in the mainstream CMA-ES assumes a different prior. 
\end{remark}

\begin{proof}
A complete proof is given in the appendix section \ref{proof:Normal inverse Wishart}. 
\end{proof}

\section{Bayesian CMA-ES}\label{sec:algorithms}

\subsection{Main assumptions}
Our main assumptions are the followings : 
\begin{itemize}
\item the parameters of the multi-variate Gaussian follow a conjugate prior distribution.
\item the minimum of our objective function $f$ follows a multi-variate normal law.
\end{itemize}

\subsection{Simulating the minimum}
One of the main challenge is to simulate the likelihood to infer the posterior. The key question is really to use the additional information of the function value $f$ for candidate points.
At step $t$ in our algorithm, we suppose multi variate Gaussian parameters $\mu$ and $\Sigma$ follow a normal inverse Wishart denoted by $NIW(\mu_{t},\kappa_{t},v_{t},\psi_{t})$. 

In full generality, we need to do a Monte Carlo of Monte Carlo as the parameters of our multi variate normal are themselves stochastic. However, we can simplify the problem and take their mean values. It is very effective in terms of computation and reduces Monte Carlo noise. For the normal inverse Wishart distribution, there exist closed form for these mean values given by:
\begin{equation}
\mathbb{E}_{t}[\mu] = \mu_{t}
\end{equation}
and
\begin{equation}
\mathbb{E}_{t}[\Sigma] = \frac{\psi_{t}}{v_{t} - n-1}
\end{equation}
We  simulate potential candidates $\mathcal{X}= \{X_i \} \sim \mathcal{N}\big(\mathbb{E}_{t}[\mu], \mathbb{E}_{t}[\Sigma] \big)$ and evaluate them $f(X_i)$. If the distribution of the minimum was accurate, the minimum would concentrate around $\mathbb{E}_{t}[\mu]$ and be spread with a variance of $\mathbb{E}_{t}[\Sigma] $. When evaluating potential candidates, as our guess is not right, we do not get values centered around 
$\mathbb{E}_{t}[\mu]$ and spread with a variance of $ \mathbb{E}_{t}[\Sigma]$. This comes from three things:
\begin{itemize}
\item Our assumed minimum is not right. We need to shift our normal to the right minimum!
\item Our assumed variance is not right. We need to compute it on real data taken into additional information given by $f$. 
\item Last but not least, our Monte Carlo simulation adds some random noise.
\end{itemize}

For the last issue, we can correct any of our estimator by the Monte Carlo bias. This can be done using standard control variate as the simulated mean and variance are given: 
$\mathbb{E}_{t}[\mu]$  and $\mathbb{E}_{t}[\Sigma]$ respectively and we can compute for each of them the bias explicitly. 

The first two issues are more complex. Let us tackle each issue one by one. 

To recover the true minimum, we design two strategies. 
\begin{itemize}
\item We design a strategy where we rebuild our normal distribution but using sorted information of our $X$'s weighted by their normal density to ensure this is a true normal corrected from the Monte Carlo bias. We need to explicitly compute the weights. 
For each simulated point $X_i$, we compute it assumed density denoted by $d_i = \mathcal{N}(\mathbb{E}_{t}[\mu], \mathbb{E}_{t}[\Sigma] )(X_i)$ where $\mathcal{N}(\mathbb{E}_{t}[\mu], \mathbb{E}_{t}[\Sigma] )(.)$ denotes the p.d.f. of the multi-variate Gaussian. 

We divide these density by their sum to get weights $(w_i)_{i=1..k}$ that are positive and sum to one as follows. $w_j = d_j / \sum_{i=1}^k d_i$. 
Hence for $k$ simulated points, we get $\{X_i, w_i \}_{i = 1..k}$. We reorder jointly the uplets (points and density) in terms of their weights in decreasing order. 

To insist we take sorted value in decreasing order with respect to the weights $(w_i)_{i=1..k}$, we denote the order statistics $(i), w\downarrow$. 

This first sorting leads to k new  uplets $ \{X_{(i), w\downarrow}, w_{(i), w\downarrow}\}_{i = 1..k}$. 
Using a \emph{stable} sort (that keeps the order of the density), we sort jointly the uplets (points and weights)  according to their objective function value (in increasing order this time) and get a k new uplets $ \{X_{(i), f\uparrow}, w_{(i), w\downarrow}\}_{i = 1..k}$. We can now compute the empirical mean $\overline{\mu}_{t}$ as follows:
\begin{equation}\label{eq:weighted_mean}
\overline{\mu}_{t}=  \underbrace{\sum_{i=1}^k  {w_{(i), w\downarrow}} \cdot  X_{(i), f\uparrow}}_{\mathrm{MC\,mean\,for\,} X_{f\uparrow}} - \underbrace{\left(\sum_{i=1}^k  w_i X_i - \overline{\mu}_{t} \right)}_{\mathrm{MC\,bias\,for\,} X}
\end{equation}
The intuition of equation \eqref{eq:weighted_mean} is to compute in the left term the Monte Carlo mean using reordered points according to their objective value and correct our initial computation by the Monte Carlo bias computed as the right term, equal to the initial Monte Carlo mean minus the real mean. We call this strategy one.

\item If we think for a minute about the strategy one, we get the intuition that when starting the minimization, it may not be optimal. This is because weights are proportional to $\exp\left\{ \frac{1}{2} (X-\mathbb{E}_{t}[\mu])^T (\mathbb{E}_{t}[\Sigma])^{-1} (X-\mathbb{E}_{t}[\mu])\right\}$. 

When we start the algorithm, we use a large search space, hence a large covariance matrix $\overline{\Sigma}_{t}$ which leads to have weights which are quite similar. Hence even if we sort candidates by their fit, ranking them according to the value of $f$ in increasing order, we will move our theoretical multi variate Gaussian little by little. A better solution is more to brutally move the center of our multi variate Gaussian to the best candidate seen so far, as follows:
\begin{equation}\label{eq:simple_mean}
\overline{\mu}_{t} = \argmin_{X \in \mathcal{X}} f(X)
\end{equation} 
We call this strategy two. Intuitively, strategy two should be best when starting the algorithm while strategy one would be better once we are close to the solution.
\end{itemize}

To recover the true variance, we can adapt what we did in stratey one as follows:
\begin{itemize}
\item \begin{eqnarray}\label{eq:weighted_cov}
 \overline{\Sigma}_{t}  = & \underbrace{{\sum_{i=1}^k  {w_{(i), w\downarrow}} \cdot  \left(X_{(i), f\uparrow} - \overline{X}_{(.), f\uparrow}\right) \left(X_{(i), f\uparrow} - \overline{X}_{(.), f\uparrow}\right)^T }{}}_{\mathrm{MC\,covariance \,for\,} X_{f\uparrow}} \nonumber \\
&\!\!\!\!\!\!\!\! - \underbrace{\left(  {\sum_{i=1}^k  {w_{i}} \! \cdot \! \left(X_{i}- \overline{X} \right) \left(X_{i} - \overline{X} \right)^T  }{} \!- \!\overline{\Sigma}_{t}   \right)}_{\mathrm{MC\,covariance\,for\,simulated\,} X} \,\,\,\,\,
\end{eqnarray}
where $ \overline{X}_{(.), f\uparrow} = \sum_{i=1}^k w_{(i), w\downarrow}X_{(i), f\uparrow}$ and $\overline{X}=\sum_{i=1}^k w_{i} X_{i}$ are respectively the mean of the sorted and non sorted points.
\item Again, we could design another strategy that takes part of the points but we leave this to further research.
\end{itemize}

Once we have the likelihood mean and variance using \eqref{eq:simple_mean} and \eqref{eq:weighted_cov} or \eqref{eq:weighted_mean} and \eqref{eq:weighted_cov}, we update the posterior law according to equation \eqref{eq:Bayesian_update}. This gives us the iterative conjugate prior parameters updates:
\begin{equation}\label{eq:Bayesian_update2}
\begin{split}
\mu_{t+1} &= \frac{\kappa_{t} \mu_{t}+ n \overline{\mu}_{t} }{\kappa_{t} +n}, \\
\kappa_{t+1}&=\kappa_{t}+ n, \\
v_{t+1} &= v_{t}+n, \\
\psi_{t+1} & = \psi_{t} \! +\!  \overline{\Sigma}_{t}  \!+\!\frac{\kappa_{t} n}{\kappa_{t} +n}\! \left( \overline{\mu}_{t}   -\mu_{t} \right) \left(  \overline{\mu}_{t}   -\mu_{t}\right)^T
\end{split}
\end{equation}

The resulting algorithm is summarized in Algo \ref{algo:BayesianCMAES}.

\begin{proposition}\label{relation_update}
Under the assumption of a NIW prior, the updates of the BCMA-ES parameters for the expected mean and variance write as a weighted combination of the prior expected mean and variance and the empirical mean and variance as follows
\begin{equation}\label{eq:weighted_sum}
\begin{split}
\mathbb{E}_{t+1}[\mu]  &= \mathbb{E}_{t}[\mu ]   + w_{t}^{\mu} \left(  \overline{\mu}_{t}  - \mathbb{E}_{t}[\mu ]\right),\\
\mathbb{E}_{t+1}[\Sigma]  &= \hspace{-0.45cm}  \underbrace{w_{t}^{\Sigma,1}}_{\text{discount factor}} \hspace{-0.5cm} \mathbb{E}_{t}[\Sigma]  +   w_{t}^{\Sigma,2}  \underbrace{\left(  \overline{\mu}_{t}   - \mathbb{E}_{t}[\mu ] \right) \left(  \overline{\mu}_{t}  - \mathbb{E}_{t}[\mu ]\right)^T}_{\text{rank one matrix}} \\
& \hspace{1.5cm} +  {w_{t}^{\Sigma,3}} \hspace{-0.5cm}   \underbrace{\overline{\Sigma}_{t}}_{\text{rank (n-1) matrix}} \\
 \mathrm{where \qquad } w_{t}^{\mu}  & =  \frac{n}{\kappa_{t} +n}, \\
 w_{t}^{\Sigma,1} & = \frac{\kappa_{t} n}{(\kappa_{t} +n)( v_{t}-1)} \\
 w_{t}^{\Sigma,2} & =  \frac{ v_{t} - n -1}{ v_{t}-1},  \\
 w_{t}^{\Sigma,3} & =  \frac{1}{ v_{t}-1}
\end{split}
\end{equation}
\end{proposition}

\begin{remark}
The proposition above is quite fundamental. It justifies that under the assumption of NIW prior, the update is a weighted sum of previous expected mean and covariance. It is striking that it provides very similar formulae to the standard CMA ES update. Recall that these updates given for the mean $m_{t}$ and covariance $C_{t}$ can be written as follows:
\begin{equation}\label{eq:std_cmaes}
\begin{split}
m_{t+1} &= m_{t}+\sum _{{i=1}}^{{\mu }}w_{i}\,(x_{{i:\lambda }}-m_{t}) \\
C_{{t+1}}&=\underbrace {(1-c_{1}-c_{\mu }+c_{s})}_{{\!\!\!\!\!{\text{discount factor}}\!\!\!\!\!}}\,C_{t}+c_{1}\underbrace {p_{c}p_{c}^{T}}_{{\!\!\!\!\!\!\!\!\!\!\!\!\!\!\!\!{\text{rank one matrix}}\!\!\!\!\!\!\!\!\!\!\!\!\!\!\!\!}} \\
& \qquad +\,c_{\mu }\underbrace {\sum _{{i=1}}^{\mu }w_{i}{\frac  {x_{{i:\lambda }}-m_{k}}{\sigma _{k}}}\left({\frac  {x_{{i:\lambda }}-m_{t}}{\sigma _{t}}}\right)^{T}}_{{{\text{rank}}\;\min(\mu ,n-1)\;{\text{matrix}}}}
\end{split}
\end{equation}
where the notations $m_{t}, w_{i},  x_{i:\lambda },C_{t}, c_{1}, c_{\mu }, c_{s}, \mathrm{etc} ...$ are given for instance in \cite{wiki:CMAES}.
\end{remark}

\begin{proof}
See \ref{proof:weighted_combination} in the appendix section.
\end{proof}

\begin{algorithm}[!ht]
\caption{Predict and Correct parameters at step t} \label{algo:BayesianCMAES}
\begin{algorithmic} [1]
\STATE \textbf{Simulate candidate}
\STATE Use mean values $\mathbb{E}_{t}[\mu] = \mu_{t}$ and $\overline{\Sigma}_{t}=\E[\Sigma] = \psi_{t}/(v_{t} - n-1)$
\STATE Simulate k points $\mathcal{X}= \{X_i \}=1..k \sim \mathcal{N}(\mathbb{E}_{t}[\mu], \overline{\Sigma}_{t})$
\STATE Compute densities $(d_i)_{i ..k} = (\mathcal{N}(\mathbb{E}_{t}[\mu],\overline{\Sigma}_{t})(X_i) )_{i ..k} = $
\STATE Sort in decreasing order with respect to $d$ to get $ \{X_{(i), d\downarrow}, d_{(i), d\downarrow}\}_{i = 1..k}$
\STATE Stable Sort in increasing order order with respect to $f(X_i)$ to get $\{X_{(i), f\uparrow}, d_{(i), d\downarrow}\}_{i = 1..k}$
\\ 
\STATE 
\STATE \textbf{Correct $\mathbb{E}_{t}[\mu]$ and $ \overline{\Sigma}_{t}$}
\STATE Either Update $\mathbb{E}_{t}[\mu]$ and $ \overline{\Sigma}_{t} $ using \eqref{eq:simple_mean} and \eqref{eq:weighted_cov}  \textbf{(strategy two)} 
\STATE Or Update $\mathbb{E}_{t}[\mu]$ and $ \overline{\Sigma}_{t} $ using \eqref{eq:weighted_mean} and \eqref{eq:weighted_cov}  \textbf{(strategy one)} 
\STATE Update $\mu_{n+1},  \kappa_{n+1}, v_{n+1}, \psi_{n+1}$ using \eqref{eq:Bayesian_update2}
\end{algorithmic}
\end{algorithm}

\subsection{Particularities of Bayesian CMA-ES}
There are some subtleties that need to be emphasized. 
\begin{itemize}
\item Although we assume a prior, we do not need to simulate the prior but can at each step use the expected value of the prior which means that we do not consume additional simulation compared to the standard CMA-ES.
\item We need to tackle local minimum (we will give example of this in the numerical section) to avoid being trapped in a bowl!  If we are in a local minimum, we need to inflate the variance to increase our search space. We do this whenever our algorithm does not manage to decrease. However, if after a while we do not get better result, we assume that this is indeed not a local minimum but rather a global minimum and start deflating the variance. This mechanism of inflation deflation ensures we can handle noisy functions like Rastrigin or Schwefel 1 or Schwefel 2 functions as defined in the section \ref{sec:experiments}.
\end{itemize}

\subsection{Differences with standard CMA-ES}
Since we use a rigorous derivation of the posterior, we have the following features:
\begin{itemize}
\item the update of the covariance takes all points. This is different from $\lambda/ \mu$ CMA-ES that uses only a subset of the point.
\item by design, the update is optimal as we compute at each step the posterior.
\item the contraction dilatation mechanism is an alternative to global local search path in standard CMA-ES.
\item weights varies across iterations which is also a major difference between main CMA ES and Bayesian CMA ES. Weights are proportional to $\exp( \frac{1}{2} X^T \Sigma^{-1} X)$ sorted in decreasing order. Initially, when the variance is large, 
\end{itemize}

\subsection{Full algorithm}
The complete Bayesian CMA ES algorithm is summarized in \ref{algo:BayesianCMAES_v2}. It iterates until a stopping condition is met. We use multiple stopping conditions. We stop if we have not increase our best result for a given number of iterations. We stop if we have reached the maximum of our iterations. We stop if our variance norm is small. Additional stopping condition can be incorporated easily. 

\begin{algorithm}[!ht]
\caption{Bayesian update of CMA-ES parameters:} \label{algo:BayesianCMAES_v2}
\begin{algorithmic} [1]
\STATE \textbf{Initialization}
\STATE Start with a prior distribution $\Pi$ on $\mu$ and $\Sigma$
\STATE Set retrial to 0
\STATE Set $f_{min}$ to max float
\WHILE{stop criteria not satisfied}
	\STATE $X \sim \mathcal{N}(\mu,\Sigma)$
	\STATE update the parameters of the Gaussian thanks to the posterior law $\Pi(\mu,\Sigma |X)$ following details given in algorithm \ref{algo:BayesianCMAES}
	\STATE Handle dilatation contraction variance for local minima as explained in  algorithm \ref{algo:DilatationContraction}
	\IF{ DilateContractFunc($X, \overline{\Sigma}_{t}, X_{min},f_{min}, \overline{\Sigma}_{t,min}$) == 1}
		\RETURN best solution
	\ENDIF
\ENDWHILE
\RETURN best solution
\end{algorithmic}
\end{algorithm}

Last but not least, we have a dilatation contraction mechanism for the variance to handle local minima with multiple level of contractions and dilatation that is given in function \ref{algo:DilatationContraction}. The overall idea is first to dilate variance if we do not make any progress to increase the search space so that we are not trapped in a local minimum. Should this not succeed, it means that we are reaching something that looks like the global minimum and we progressively contract the variance. In our implemented algorithm, we take $L_1=5, L_2=20, L_3=30, L_4=40, L_5=50$ and the dilatation, contraction parameters given by $k_1=1.5, k_2=0.9, k_3=0.7, k_5=0.5$ We have also a restart at previous minimum level $L_{*}=L_2$.

\begin{algorithm}[!ht]
\caption{Dilatation contraction variance for local minima:} \label{algo:DilatationContraction}
\begin{algorithmic}[1]
\STATE \textbf{Function} DilateContractFunc($X, \overline{\Sigma}_{t}, X_{min},f_{min},\overline{\Sigma}_{t,min}$)
	\IF{$f(X) \leq f_{min}$}	
		\STATE Set $f_{min}= f(X)$ 
		\STATE Memorize current point and its variance: 
		\STATE \hspace{0.3cm}  {\tiny$\bullet$} $X_{min}= X$
		\STATE \hspace{0.3cm}  {\tiny$\bullet$} $\overline{\Sigma}_{t,min}= \overline{\Sigma}_{t}$
		\STATE Set retrial = 0
	\ELSE
		\STATE Set retrial += 1
		\IF{retrial == $L_{*}$}
			\STATE Restart at previous best solution: 
			\STATE \hspace{0.2cm} {\tiny$\bullet$} $X = X_{min}$
			\STATE \hspace{0.2cm} {\tiny$\bullet$} $\overline{\Sigma}_{t}= \overline{\Sigma}_{t,min}$
		\ENDIF

		\IF{$L_2$ > retrial and retrial > $L_1$}
			\STATE Dilate variance by $k_1$
			
		\ELSIF{$L_3$ > retrial and retrial $\geq$ $L_2$}
			\STATE Contract variance by $k_2$
		\ELSIF{$L_4$ > retrial and retrial $\geq$  $L_3$}
			\STATE Contract variance by $k_3$
		\ELSIF{$L_5$ > retrial and retrial $\geq$  $L_4$}
			\STATE Contract variance by $k_4$
		\ELSE
			\RETURN 1
		\ENDIF
	\RETURN 0
	\ENDIF
\STATE \textbf{End Function}
\end{algorithmic}
\end{algorithm}

\section{Numerical results}\label{sec:experiments}
\subsection{Functions examined}
We have examined four functions to stress test our algorithm. They are listed in increasing order of complexity for our algorithm and correspond to different type of functions. They are all generalized function that can defined for any dimension $n$. For all, we present the corresponding equation for a variable $x=(x_1,x_2, .., x_n)$ of $n$ dimension. Code is provided in supplementary materials. We have frozen seeds to have \emph{reproducible results}.

\subsubsection{Cone}
The most simple function to optimize is the quadratic cone whose equation is given by \eqref{eq:function_Cone} and represented in figure \ref{fig:cone}. It is also the standard Euclidean norm. It is obviously convex and is a good test of the performance of an optimization method.
\begin{equation}\label{eq:function_Cone}
f(x)= \left( \sum_{i=1}^n x_i^2 \right)^{1/2} = \| x \|_2
\end{equation}

\begin{figure}[!ht]
\centering
\includegraphics[width=5.5cm]{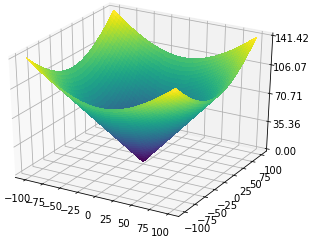}
\caption{A simple convex function: the quadratic norm. Minimum in 0} \label{fig:cone}
\end{figure}

\subsubsection{Schwefel 2 function}
A slightly more complicated function is the Schwefel 2 function whose equation is given by \eqref{eq:function_Schwefel2} and represented in figure \ref{fig:Schwefel2}. It is a piecewise linear function and validates the algorithm can cope with non convex function.

\begin{equation}\label{eq:function_Schwefel2}
f(x)= \sum_{i=1}^n \mid x_i\mid + \prod_{i=1}^n\mid x_i \mid
\end{equation} 

\begin{figure}[!ht]
\centering
\includegraphics[width=5.5cm]{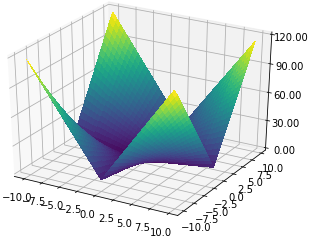}
\caption{Schwefel 2 function: a simple piecewise linear function} \label{fig:Schwefel2}
\end{figure}

\subsubsection{Rastrigin}
The Rastrigin function, first proposed by \cite{Rastrigin_1974} and generalized by \cite{Muhlenbein_1991}, is more difficult compared to the Cone and the Schwefel 2 function. Its equation is given by \eqref{eq:function_Rastrigin} and represented in figure \ref{fig:Rastrigin}. It is a non-convex function often used as a performance test problem for optimization algorithms. It is a typical example of non-linear multi modal function. Finding its minimum is considered a good stress test for an optimization algorithm, due to its large search space and its large number of local minima. 

\begin{figure}[!ht]
\centering
\includegraphics[width=5.5cm]{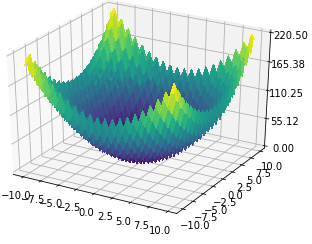}
\caption{Rastrigin function: a non convex function multi-modal and with a large number of local minima} \label{fig:Rastrigin}
\end{figure}

\begin{equation}\label{eq:function_Rastrigin}
 f(x)=10 \times n + \sum_{i=1}^n \left[ x_i^2-10 \cos(2\pi x_i)\right]
\end{equation}

\subsubsection{Schwefel 1 function}
The last function we tested is the Schwefel 1 function whose equation is given by \eqref{eq:function_Schwefel1} and represented in figure \ref{fig:Schwefel1}. It is sometimes only defined on $\left[-500, 500 \right]^n$. The Schwefel 1 function shares similarities with the Rastrigin function. It is continuous, not convex, multi-modal and with a large number of local minima. The extra difficulty compared to the Rastrigin function, the local minima are more pronounced local bowl making the optimization even harder.

\begin{eqnarray}\label{eq:function_Schwefel1}
&& \hspace{-0.5cm} f(x) = 418.9829\, \times \, n \nonumber \\
&& - \sum_{i=1}^n  \left[ x_i \sin(\sqrt{\mid x_i \mid}) \mathbbm{1}_{|x_i|< 500} + 500 \sin(\sqrt{500})  \mathbbm{1}_{|x_i| \geq 500} \right] \qquad
\end{eqnarray}

\begin{figure}[!ht]
\centering
\includegraphics[width=5.5cm]{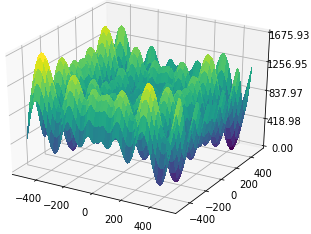}
\caption{Schwefel 1 function: a non convex function multi-modal and with a large number of local pronounced bowls} \label{fig:Schwefel1}
\end{figure}

\subsection{Convergence}
For each of the functions, we compared our method using strategy one entitled \emph{B-CMA-ES S1}: update $\overline{\mu}_{t}$ and $ \overline{\Sigma}_{t} $ using \eqref{eq:weighted_mean} and \eqref{eq:weighted_cov} plotted in \emph{orange}, or strategy two \emph{B-CMA-ES S2}: same update but using \eqref{eq:simple_mean} and \eqref{eq:weighted_cov}, plotted in \emph{blue} and standard CMA-ES as provided by the opensource python package pycma plotted in \emph{green}. We clearly see that strategies one and two are quite similar to standard CMA-ES. The convergence graphics that show the error compared to the minimum are represented: 
\begin{itemize}
\item for the cone function by figure \ref{fig:convergence_cone} (case of a convex function), with initial point $(10,10)$
\item for the Schwefel 2 function in figure \ref{fig:convergence_Schwefel2} (case of piecewise linear function), with initial point $(10,10)$
\item for the Rastrigin function in figure \ref{fig:convergence_Rastrigin} (case of a non convex function with multiple local minima), with initial point $(10,10)$
\item and for the Schwefel 1 function in figure \ref{fig:convergence_Schwefel1} (case of a non convex function with multiple large bowl local minima), with initial point $(400,400)$
\end{itemize}
The results are for one test run. In a forthcoming paper, we will benchmark them with more runs to validate the interest of this new method.
\begin{figure}[!ht]
\centering
\includegraphics[height=4.5cm]{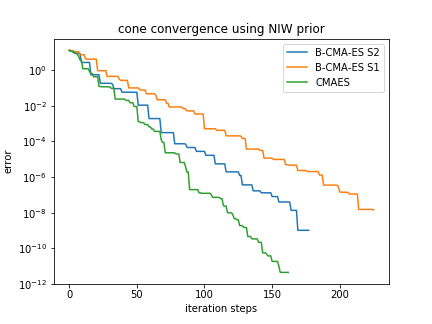}
\caption{Convergence for the Cone function}
\label{fig:convergence_cone}
\end{figure}

\begin{figure}[!ht]
\centering
\includegraphics[height=4.5cm]{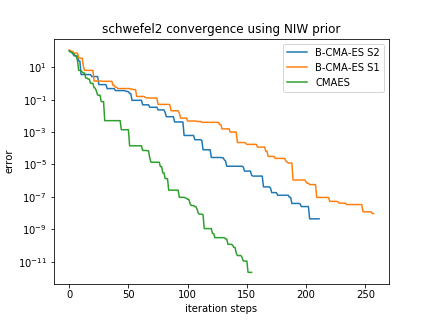}
\caption{Convergence for the Schwefel 2 function}
\label{fig:convergence_Schwefel2}
\end{figure}

\begin{figure}[!ht]
\centering
\includegraphics[height=4.5cm]{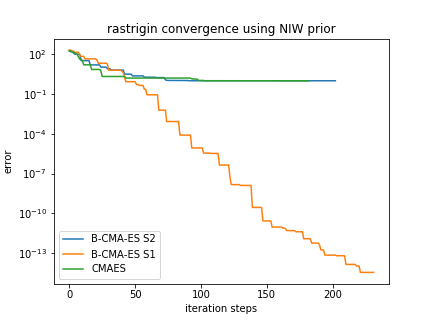}
\caption{Convergence for the Rastrigin function}
\label{fig:convergence_Rastrigin}
\end{figure}

\begin{figure}[!ht]
\centering
\includegraphics[height=4.5cm]{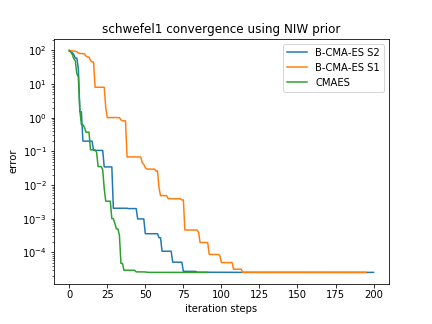}
\caption{Convergence for the Schwefel 1 function}
\label{fig:convergence_Schwefel1}
\end{figure}

For the four functions, BCMAES achieves convergence similar to standard CMA-ES. The intuition of this good convergence is that shifting the multi variate mean by the best candidate seen so far is a good guess to update it at the next run (standard CMA-ES or B-CMA-ES S1).

\section{Conclusion}
In this paper, we have revisited the CMA-ES algorithm and provided a Bayesian version of it. Taking conjugate priors, we can find optimal update for the mean and covariance of the multi variate Normal. We have provided the corresponding algorithm that is a new version of CMA-ES. First numerical experiments show this new version is competitive to standard CMA-ES on traditional functions such as cone, Schwefel 1, Rastrigin and Schwefel 2. This faster convergence can be explained on a theoretical side from an optimal update of the prior (thanks to Bayesian update) and the use of the best candidate seen at each simulation to shift the mean of the multi-variate Gaussian likelihood. We envisage further works to benchmark our algorithm to other standard evolutionary algorithms, in particular to use the COCO platform to provide more meaningful tests and confirm the theoretical intuition of good performance of this new version of CMA-ES, and to test the importance of the prior choice. 

\section{Appendix}
\subsection{Conjugate priors}\label{proof:conjugate_priors}
\begin{proof}
Consider $n$ independent and identically distributed (IID) measurements $\mathcal{X}\triangleq\{ \x^j\in \mathbb{R}^d | 1\leq j \leq n\}$ and assume that these variables have an exponential family density. The likelihood $p(\mathcal{X}|\theta,\lambda)$ writes simply as the product of each individual likelihood:
\begin{align}
\hspace{-0.3cm} p(\mathcal{X}|\theta,\lambda) \! = \!\Big(\prod_{j=1}^{n}h(\x^j)\Big) \exp\Big(\eta(\theta,\lambda)^T \sum_{j=1}^n T(x^j)-n A(\eta(\theta,\lambda))\Big).  \label{eq:ExponentialFamilyLikelihood}
\end{align}
If we start with a prior $\pi(\theta)$ of the form $\pi( \theta )\propto \exp (\mathcal{F}(\theta))$ for some function $\mathcal{F}(\cdot)$, its posterior writes: 
\begin{align}
\pi( \theta | \mathcal{X} ) &\propto p( \mathcal{X} | \theta )  \exp (\mathcal{F}(\theta)) \nonumber \\
&\propto  \exp \left( \eta(\theta, \lambda) \cdot \sum_{j=1}^n T(x^j)-n A(\eta(\theta,\lambda)) +\mathcal{F}(\theta)\right).\label{eq:ExponentialFamilyPosterior-expanded}
\end{align}
It is easy to check that the posterior \eqref{eq:ExponentialFamilyPosterior-expanded} is in the same exponential family as the prior iff $\mathcal{F}(\cdot)$ is in the form:
\begin{align}
\mathcal{F}(\theta)= \mu_1\cdot \eta(\theta,\lambda) - \mu_0 A(\eta(\theta,\lambda))
\end{align} 
for some $\mu\triangleq(\mu_0, \mu_1)$, such that: 
\begin{align}
\!\!\!\! p(\mathcal{X}|\theta,\lambda)  \! \propto \! \exp\Big(\Big(\mu_1+ \sum_{j=1}^n T(x^j)\Big)^T \!\!  \eta(\theta,\lambda)  -( n +\mu_0) A(\eta(\theta,\lambda)) \Big)\! .\label{eq:ExponentialFamilyPosteriorLinearcombination}
\end{align}
Hence, the conjugate prior for the likelihood \eqref{eq:ExponentialFamilyLikelihood} is parametrized by $\mu$ and  given by: 
\begin{align}
p(\mathcal{X}|\theta,\lambda)  =\frac{1}{Z}\exp\left( \mu_1\cdot \eta(\theta,\lambda) - \mu_0 A(\eta(\theta,\lambda)) \right),
\end{align} 
where $Z={\int{\exp\left(\mu_1\cdot \eta(\theta,\lambda) - \mu_0 A(\eta(\theta,\lambda)) \right) \d x}}$.
\end{proof}

\subsection{Exact computation of the posterior update for the Normal inverse Wishart}\label{proof:Normal inverse Wishart}
To make our proof simple, we first start by the one dimensional case and show that in one dimension it is a normal inverse gamma. We then generalize to the multi dimensional case.
\begin{lemma} \label{lemma1}
The probability density function of a Normal inverse gamma (denoted by NIG) random variable can be expressed as the product of a Normal and an Inverse gamma probability density functions.
\end{lemma}
\begin{proof} we suppose that $x|\mu,\sigma^2 \sim \mathcal{N}(\mu_0,\sigma^2/v)$. We recall the following definition of conditional probability:
\begin{definition}\label{def1}
Suppose that events A,B and C are defined on the same probability space, and the event B is such that $\mathbb{P}(B) > 0$. We have the following expression:\\
$\mathbb{P}(A \cap B|C) = \mathbb{P}(A|B,C) \mathbb{P}(B|C)$.
\end{definition}
Applying \ref{def1}, we have: 
\begin{flalign}
p\left( \mu, \sigma^2 | \mu_0,v,\alpha,\beta \right) & = p\left( \mu | \sigma^2, \mu_0,v,\alpha,\beta \right) p\left( \sigma^2 | \mu_0,v,\alpha,\beta \right) \nonumber \\
&= p\left( \mu | \sigma^2, \mu_0,v \right) p\left( \sigma^2 | \alpha,\beta \right). 
\end{flalign}
Using the definition of the Normal inverse gamma law, we end the proof.
\end{proof}

\begin{remark}\label{remark1}
If $\left( x,\sigma^2\right) \sim NIG(\mu,\lambda,\alpha,\beta)$, the probability density function is the following: 
\begin{flalign}
f( x,\sigma^2 | \mu,\lambda,\alpha,\beta) & = \frac{\sqrt{\lambda}}{\sigma \sqrt{2\pi}} \frac{\beta^{\alpha}}{\Gamma(\alpha)} \left(\frac{1}{\sigma^2}\right)^{\alpha+1} \nonumber \\
& \hspace{2cm} \exp\left\{- \frac{2\beta+\lambda (x-\mu)^2}{2\sigma^2} \right\}.
\end{flalign}    
\end{remark}

\begin{proposition} \label{proposition1}
The Normal Inverse Gamma NIG $\left( \mu_0,v,\alpha,\beta \right) $ distribution is a conjugate prior of a normal distribution with unknown mean and variance.
\end{proposition}
\begin{proof} the posterior is proportional to the product of the prior and likelihood, then:
\begin{flalign}
p\left(\mu, \sigma^2|X \right) &\propto \frac{\sqrt{v}}{ \sqrt{2\pi}} \left(\frac{1}{\sigma^2}\right)^{1/2} \exp\left\{ \frac{-v(\mu-\mu_0)^2}{2\sigma^2} \right\} \nonumber \\
&\hspace{0.5cm} \times \frac{\beta^{\alpha}}{\Gamma(\alpha)}  \left(\frac{1}{\sigma^2}\right)^{\alpha+1} \exp\left\{ \frac{-\beta}{\sigma^2} \right\} \nonumber \\
&\hspace{0.5cm} \times \left(\frac{1}{ 2\pi \sigma^2}\right)^{n/2} \exp\left\{- \frac{\sum_{i=1}^n \left( x_i-\mu \right)^2}{2\sigma^2} \right\}.
\end{flalign}

Defining the empirical mean and variance as $\overline{x} = \frac{1}{n} \sum_{i=1}^n x_i $ and $\overline{s} = \frac{1}{n} \sum_{i=1}^n (x_i -\overline{x})^2$, we obtain that $\sum_{i=1}^n \left( x_i-\mu \right)^2 = n(\overline{s}+(\overline{x}-\mu)^2)$.
\vspace{0.3cm}

\noindent So, the conditional density writes:
\begin{flalign}
 p\left(\mu, \sigma^2|X \right) &\propto \sqrt{v}  \left(\frac{1}{\sigma^2}\right)^{\alpha+n/2+3/2} \nonumber \\
&\hspace{-1cm} \times  \exp\Bigg\{- \frac{1}{\sigma^2} \Big[ \beta + \frac{1}{2} \Big( v(\mu-\mu_0)^2 +n\big(\overline{s}+(\overline{x}-\mu)^2\big) \Big) \Big] \Bigg\}.
\end{flalign}

Besides,
\begin{flalign}
 &\hspace{0.5cm} v(\mu-\mu_0)^2 +n\big(\overline{s}+(\overline{x}-\mu)^2\big)  \nonumber \\
& =  v \big(\mu^2-2\mu \mu_0+\mu_o^2 \big)+n \overline{s}+n\big(\overline{x}^2-2\overline{x}\mu+\mu^2\big)  \nonumber \\
& =  \mu^2 \left(v+n \right) -2\mu \left(v\mu_0+n\overline{x} \right)+v\mu_o^2+n\overline{s}+n\overline{x}^2.
\end{flalign}

Denoting $a=v+n$ and $b=v\mu_0+n\overline{x}$, we have :
\begin{flalign}
& \hspace{0.5cm} \beta + \frac{1}{2} \Big( v(\mu-\mu_0)^2 +n\big(\overline{s}+(\overline{x}-\mu)^2\big)\Big)  \nonumber \\
&=\beta +\frac{1}{2} \left(a \mu^2  -2b\mu +v\mu_o^2+n\overline{s}+n\overline{x}^2\right)  \nonumber \\
&=\beta +\frac{1}{2} \left(a\left( \mu^2  -\frac{2b}{a}\mu\right) +v\mu_o^2+n\overline{s}+n\overline{x}^2\right)  \nonumber \\
&=\beta +\frac{1}{2} \left(a\left( \mu  -\frac{b}{a}\right)^2 - \frac{b^2}{a} +v\mu_o^2+n\overline{s}+n\overline{x}^2\right).
\end{flalign}
So we can express the proportional expression of the posterior :
\begin{flalign*}
p\left(\mu, \sigma^2|X \right) \propto  \left(\frac{1}{\sigma^2}\right)^{\alpha^\star+3/2} \times  \exp\Bigg\{- \frac{2\beta^\star+\lambda^\star \left( \mu  -\mu^\star\right)^2}{2\sigma^2}\Bigg\},
\end{flalign*}

with
\begin{itemize}
\item $\alpha^\star = \alpha+\frac{n}{2}$
\item$ \beta^\star=\beta +\frac{1}{2}\left(\sum_{i=1}^n \left(x_i-\overline{x}\right)^2  +\frac{nv}{n+v} \frac{(\overline{x}-\mu_0)^2}{2}\right)$
\item $\mu^\star =\frac{v\mu_0+n\overline{x}}{v+n} $
\item $\lambda^\star = v+n$
\end{itemize} 
We can identify the terms with the expression of the probability density function given in \ref{remark1} to conclude that the posterior follows a NIG$(\mu^\star,\lambda^\star,\alpha^\star,\beta^\star)$.
\end{proof}

\noindent We are now ready to prove the following proposition:
\begin{proposition} \label{proposition2}
The Normal Inverse Wishart  (denoted by NIW) $\left( \mu_0,\kappa_0,v_0,\psi \right) $ distribution is a conjugate prior of a multivariate normal distribution with unknown mean and covariance.
\end{proposition}

\begin{proof} we use the fact that the probability density function of a Normal inverse Wishart random variable can be expressed as the product of a Normal and an Inverse Wishart probability density functions (we use the same reasoning that in \ref{lemma1}). Besides, the posterior is proportional to the product of the prior and the likelihood. \\
We first express the probability density function of the multivariate Gaussian random variable in a proper way in order to use it when we write the posterior density function.
\begin{eqnarray}
 & \sum_{i=1}^n \left( x_i-\mu\right)^T \Sigma^{-1} \left( x_i-\mu\right) \nonumber \\
 =& n \left( \overline{x}-\mu\right)^T \Sigma^{-1} \left( \overline{x}-\mu\right)  + \sum_{i=1}^n \left( x_i-\overline{x}\right)^T \Sigma^{-1} \left( x_i-\overline{x}\right).
\end{eqnarray}
We can inject the previous result and use the properties of the trace function to express the following probability density function of the multivariate Gaussian random variable of parameters $\mu$ and $\Sigma$. The density writes as:

\begin{eqnarray}
&  \hspace{-2cm} \frac{|\Sigma|^{-n/2}}{ \sqrt{(2\pi)^{pn}}} \exp\Bigg\{ -\frac{n}{2} \left( \overline{x}-\mu\right)^T \Sigma^{-1} \left( \overline{x}-\mu \right)\nonumber \\
&  \hspace{2cm}  -\frac{1}{2} tr\left(\Sigma^{-1} \sum_{i=1}^n \left( x_i-\overline{x}\right) \left( x_i-\overline{x}\right)^T\right) \Bigg\}.
\end{eqnarray}

\noindent Hence, we can compute explicitly the posterior as follows:
\begin{flalign}
p\left(\mu, \sigma^2|X \right) &\propto \frac{\sqrt{\kappa_0}}{ \sqrt{(2\pi)^p|\Sigma|}} \exp\left\{- \frac{\kappa_0}{2} \big(\mu-\mu_0 \big)^T\Sigma^{-1} \big(\mu-\mu_0 \big) \right\}\nonumber \\
& \hspace{0.5cm} \times \frac{|\psi|^{v/2}}{ 2^{vp/2}\Gamma_p(v_0/2)} |\Sigma|^{-\frac{v_0+p+1}{2}}exp\left\{- \frac{1}{2} tr\left(\psi \Sigma^{-1} \right) \right\} \nonumber \\
& \hspace{0.5cm} \times |\Sigma|^{-n/2} \exp \Bigg\} -\frac{n}{2} \left( \overline{x}-\mu\right)^T \Sigma^{-1} \left( \overline{x}-\mu \right) \nonumber \\
& \hspace{1 cm} -\frac{1}{2} tr\left(\Sigma^{-1} \sum_{i=1}^n \left( x_i-\overline{x}\right) \left( x_i-\overline{x}\right)^T\right) \Bigg\}  
\end{flalign}
\begin{flalign}
&\propto  |\Sigma|^{-\frac{v_0+p+2+n}{2}} \exp\Bigg\{- \frac{\kappa_0}{2} \big(\mu-\mu_0 \big)^T\Sigma^{-1} \big(\mu-\mu_0 \big)  \nonumber \\
& \hspace{0.5cm} -\frac{n}{2} \left( \overline{x}-\mu\right)^T \Sigma^{-1} \left( \overline{x}-\mu \right) \nonumber \\
& \hspace{0.5cm} -\frac{1}{2} tr\left(\Sigma^{-1} \Big(\psi + \sum_{i=1}^n \left( x_i-\overline{x}\right) \left( x_i-\overline{x}\right)^T\Big)\right) \Bigg\} .
\end{flalign}

We organize the terms and find the parameters of our Normal Inverse Wishart random variable NIW$(\mu_0^\star,\kappa_0^\star,v_0^\star,\psi^\star)$.
\begin{equation}
\begin{split}
\mu_0^\star &= \frac{\kappa_0 \mu_0+n\overline{x}}{\kappa_0 +n}, \quad \kappa_0^\star =\kappa_0+n, \quad v_0^\star  = v_0+n \\
\psi^\star &= \psi + \sum_{i=1}^n \left( x_i-\overline{x}\right) \left( x_i-\overline{x}\right)^T+\frac{\kappa_0 n}{\kappa_0 +n}\left(\overline{x}-\mu_0\right) \left(\overline{x}-\mu_0\right)^T
\end{split}
\end{equation}
which are exactly the equations provided in \eqref{eq:Bayesian_update}.
\end{proof}

\subsection{Weighted combination for the BCMA ES update}\label{proof:weighted_combination}
\begin{proof}
\begin{equation}
\begin{split}
\mathbb{E}_{t+1}[\mu]  &= \mu_{t+1} \\
& =  \frac{\kappa_{t} \mu_{t}+ n \overline{\mu}_{t} }{\kappa_{t} +n}\\
& =  \mathbb{E}_{t}[\mu ]   + w_{t}^{\mu} \left( \hat{\mu} - \mathbb{E}_{t}[\mu ]\right)
\end{split}
\end{equation}

\begin{equation}
\begin{split}
\mathbb{E}_{t+1}[\Sigma]  & = \frac{\psi_{t+1}}{v_{t+1}-n-1} \\
& =  \frac{1}{v_{t}-1}\psi_{t} \! +\! \frac{1}{v_{t}-1} \overline{\Sigma}_{t}  \! \\
& \quad +\!\frac{\kappa_{t} n}{(\kappa_{t} +n)(v_{t}-1)}\! \left( \overline{\mu}_{t}   -\mu_{t} \right) \left(  \overline{\mu}_{t}   -\mu_{t}\right)^T \\
&= \hspace{-0.45cm}  \underbrace{w_{t}^{\Sigma,1}}_{\text{discount factor}} \hspace{-0.5cm} \mathbb{E}_{t}[\Sigma]  +   w_{t}^{\Sigma,2}  \underbrace{\left(\hat{\mu} - \mathbb{E}_{t}[\mu ] \right) \left(\hat{\mu} - \mathbb{E}_{t}[\mu ]\right)^T}_{\text{rank one matrix}} \\
& \quad +{w_{t}^{\Sigma,3}} \hspace{-0.5cm}   \underbrace{\overline{\Sigma}_{t}}_{\text{rank (n-1) matrix}} \\
 \mathrm{where \qquad } w_{t}^{\mu}  & =  \frac{n}{\kappa_{t} +n}, \\
 w_{t}^{\Sigma,1} & =  \frac{ v_{t} - n -1}{ v_{t}-1}, \\
 w_{t}^{\Sigma,2} & = \frac{\kappa_{t} n}{(\kappa_{t} +n)( v_{t}-1)}
\end{split}
\end{equation}

\noindent $\overline{\Sigma}_{t}$ is a covariance matrix of rank $n-1$ as we subtract the empirical mean (which removes one degree of freedom). The matrix ${\left(\hat{\mu} - \mathbb{E}_{t}[\mu ] \right) \left(\hat{\mu} - \mathbb{E}_{t}[\mu ]\right)^T}$ is of rank 1 as it is parametrized by the vector $\hat{\mu}$.
\end{proof}

\clearpage
\normalsize
\bibliographystyle{ACM-Reference-Format}
\bibliography{bibfile} 


\begin{thebibliography}{26}


\ifx \showCODEN    \undefined \def \showCODEN     #1{\unskip}     \fi
\ifx \showDOI      \undefined \def \showDOI       #1{#1}\fi
\ifx \showISBNx    \undefined \def \showISBNx     #1{\unskip}     \fi
\ifx \showISBNxiii \undefined \def \showISBNxiii  #1{\unskip}     \fi
\ifx \showISSN     \undefined \def \showISSN      #1{\unskip}     \fi
\ifx \showLCCN     \undefined \def \showLCCN      #1{\unskip}     \fi
\ifx \shownote     \undefined \def \shownote      #1{#1}          \fi
\ifx \showarticletitle \undefined \def \showarticletitle #1{#1}   \fi
\ifx \showURL      \undefined \def \showURL       {\relax}        \fi
\providecommand\bibfield[2]{#2}
\providecommand\bibinfo[2]{#2}
\providecommand\natexlab[1]{#1}
\providecommand\showeprint[2][]{arXiv:#2}

\bibitem[\protect\citeauthoryear{Akimoto, Auger, and Hansen}{Akimoto
  et~al\mbox{.}}{2015}]%
        {Auger_2015}
\bibfield{author}{\bibinfo{person}{Youhei Akimoto}, \bibinfo{person}{Anne
  Auger}, {and} \bibinfo{person}{Nikolaus Hansen}.}
  \bibinfo{year}{2015}\natexlab{}.
\newblock \showarticletitle{Continuous Optimization and {CMA-ES}}.
\newblock \bibinfo{journal}{\emph{{GECCO} 2015, Madrid, Spain}}
  \bibinfo{volume}{1} (\bibinfo{year}{2015}), \bibinfo{pages}{313--344}.
\newblock


\bibitem[\protect\citeauthoryear{Akimoto, Auger, and Hansen}{Akimoto
  et~al\mbox{.}}{2016}]%
        {Auger_2016}
\bibfield{author}{\bibinfo{person}{Youhei Akimoto}, \bibinfo{person}{Anne
  Auger}, {and} \bibinfo{person}{Nikolaus Hansen}.}
  \bibinfo{year}{2016}\natexlab{}.
\newblock \showarticletitle{{CMA-ES} and Advanced Adaptation Mechanisms}.
\newblock \bibinfo{journal}{\emph{{GECCO}, Denver}}  \bibinfo{volume}{2016}
  (\bibinfo{year}{2016}), \bibinfo{pages}{533--562}.
\newblock


\bibitem[\protect\citeauthoryear{Akimoto, Nagata, Ono, and Kobayashi}{Akimoto
  et~al\mbox{.}}{2010}]%
        {Akimoto_2010}
\bibfield{author}{\bibinfo{person}{Youhei Akimoto}, \bibinfo{person}{Yuichi
  Nagata}, \bibinfo{person}{Isao Ono}, {and} \bibinfo{person}{Shigenobu
  Kobayashi}.} \bibinfo{year}{2010}\natexlab{}.
\newblock \showarticletitle{Bidirectional Relation between CMA Evolution
  Strategies and Natural Evolution Strategies}.
\newblock \bibinfo{journal}{\emph{PPSN}} \bibinfo{volume}{XI},
  \bibinfo{number}{1} (\bibinfo{year}{2010}), \bibinfo{pages}{154--163}.
\newblock


\bibitem[\protect\citeauthoryear{Auger and Hansen}{Auger and Hansen}{2009}]%
        {Auger_2009}
\bibfield{author}{\bibinfo{person}{Anne Auger} {and} \bibinfo{person}{Nikolaus
  Hansen}.} \bibinfo{year}{2009}\natexlab{}.
\newblock \showarticletitle{Benchmarking the {(1+1)-CMA-ES} on the {BBOB-2009}
  noisy testbed}.
\newblock \bibinfo{journal}{\emph{Companion Material}}
  \bibinfo{volume}{{GECCO} 2009} (\bibinfo{year}{2009}),
  \bibinfo{pages}{2467--2472}.
\newblock


\bibitem[\protect\citeauthoryear{Auger and Hansen}{Auger and Hansen}{2012}]%
        {Auger_2012}
\bibfield{author}{\bibinfo{person}{Anne Auger} {and} \bibinfo{person}{Nikolaus
  Hansen}.} \bibinfo{year}{2012}\natexlab{}.
\newblock \showarticletitle{Tutorial {CMA-ES:} evolution strategies and
  covariance matrix adaptation}.
\newblock \bibinfo{journal}{\emph{Companion Material Proceedings}}
  \bibinfo{volume}{2012}, \bibinfo{number}{12} (\bibinfo{year}{2012}),
  \bibinfo{pages}{827--848}.
\newblock


\bibitem[\protect\citeauthoryear{Auger, Schoenauer, and Vanhaecke}{Auger
  et~al\mbox{.}}{2004}]%
        {Auger_2004}
\bibfield{author}{\bibinfo{person}{Anne Auger}, \bibinfo{person}{Marc
  Schoenauer}, {and} \bibinfo{person}{Nicolas Vanhaecke}.}
  \bibinfo{year}{2004}\natexlab{}.
\newblock \showarticletitle{{LS-CMA-ES:} {A} Second-Order Algorithm for
  Covariance Matrix Adaptation}.
\newblock \bibinfo{journal}{\emph{{PPSN} VIII, 8th International Conference,
  Birmingham, UK, September 18-22, 2004, Proceedings}} \bibinfo{volume}{2004},
  \bibinfo{number}{2004} (\bibinfo{year}{2004}), \bibinfo{pages}{182--191}.
\newblock


\bibitem[\protect\citeauthoryear{Gelman, Carlin, Stern, and Rubin}{Gelman
  et~al\mbox{.}}{2004}]%
        {Gelman_2004}
\bibfield{author}{\bibinfo{person}{Andrew Gelman}, \bibinfo{person}{John~B.
  Carlin}, \bibinfo{person}{Hal~S. Stern}, {and} \bibinfo{person}{Donald~B.
  Rubin}.} \bibinfo{year}{2004}\natexlab{}.
\newblock \bibinfo{booktitle}{\emph{Bayesian Data Analysis}
  (\bibinfo{edition}{2nd ed.} ed.)}.
\newblock \bibinfo{publisher}{Chapman and Hall/CRC}, \bibinfo{address}{New
  York}.
\newblock


\bibitem[\protect\citeauthoryear{Glasmachers, Schaul, Yi, Wierstra, and
  Schmidhuber}{Glasmachers et~al\mbox{.}}{2010}]%
        {Glasmachers_2010}
\bibfield{author}{\bibinfo{person}{T. Glasmachers}, \bibinfo{person}{T.
  Schaul}, \bibinfo{person}{S. Yi}, \bibinfo{person}{D. Wierstra}, {and}
  \bibinfo{person}{J.: Schmidhuber}.} \bibinfo{year}{2010}\natexlab{}.
\newblock \showarticletitle{Exponential natural evolution strategies}.
\newblock \bibinfo{journal}{\emph{In: Proceedings of Genetic and Evolutionary
  Computation Conference, pp}} \bibinfo{volume}{2010}, \bibinfo{number}{2010}
  (\bibinfo{year}{2010}), \bibinfo{pages}{393--400}.
\newblock


\bibitem[\protect\citeauthoryear{Hansen}{Hansen}{2016}]%
        {Hansen_2016}
\bibfield{author}{\bibinfo{person}{Nikolaus Hansen}.}
  \bibinfo{year}{2016}\natexlab{}.
\newblock \bibinfo{title}{The {CMA} Evolution Strategy: {A} Tutorial,
  Preprint}.
\newblock
\newblock
\showeprint{1604.00772}


\bibitem[\protect\citeauthoryear{Hansen and Auger}{Hansen and Auger}{2011}]%
        {Hansen_2011}
\bibfield{author}{\bibinfo{person}{Nikolaus Hansen} {and} \bibinfo{person}{Anne
  Auger}.} \bibinfo{year}{2011}\natexlab{}.
\newblock \showarticletitle{{CMA-ES:} evolution strategies and covariance
  matrix adaptation}.
\newblock \bibinfo{journal}{\emph{{GECCO} 2011}} \bibinfo{volume}{2011},
  \bibinfo{number}{1} (\bibinfo{year}{2011}), \bibinfo{pages}{991--1010}.
\newblock


\bibitem[\protect\citeauthoryear{Hansen and Auger}{Hansen and Auger}{2014}]%
        {Hansen_2014}
\bibfield{author}{\bibinfo{person}{Nikolaus Hansen} {and} \bibinfo{person}{Anne
  Auger}.} \bibinfo{year}{2014}\natexlab{}.
\newblock \showarticletitle{Evolution strategies and {CMA-ES} (covariance
  matrix adaptation)}.
\newblock \bibinfo{journal}{\emph{{GECCO} Vancouver}} \bibinfo{volume}{2014},
  \bibinfo{number}{14} (\bibinfo{year}{2014}), \bibinfo{pages}{513--534}.
\newblock


\bibitem[\protect\citeauthoryear{Hansen and Ostermeier}{Hansen and
  Ostermeier}{2001}]%
        {HansenOstermeier_2001}
\bibfield{author}{\bibinfo{person}{Nikolaus Hansen} {and}
  \bibinfo{person}{Andreas Ostermeier}.} \bibinfo{year}{2001}\natexlab{}.
\newblock \showarticletitle{Completely Derandomized Self-Adaptation in
  Evolution Strategies}.
\newblock \bibinfo{journal}{\emph{Evolutionary Computation}}
  \bibinfo{volume}{9}, \bibinfo{number}{2} (\bibinfo{year}{2001}),
  \bibinfo{pages}{159--195}.
\newblock


\bibitem[\protect\citeauthoryear{Heidrich{-}Meisner and
  Igel}{Heidrich{-}Meisner and Igel}{2009}]%
        {Igel_2009b}
\bibfield{author}{\bibinfo{person}{Verena Heidrich{-}Meisner} {and}
  \bibinfo{person}{Christian Igel}.} \bibinfo{year}{2009}\natexlab{}.
\newblock \showarticletitle{Neuroevolution strategies for episodic
  reinforcement learning}.
\newblock \bibinfo{journal}{\emph{J. Algorithms}} \bibinfo{volume}{64},
  \bibinfo{number}{4} (\bibinfo{year}{2009}), \bibinfo{pages}{152--168}.
\newblock


\bibitem[\protect\citeauthoryear{Igel}{Igel}{2010}]%
        {Igel_2010}
\bibfield{author}{\bibinfo{person}{Christian Igel}.}
  \bibinfo{year}{2010}\natexlab{}.
\newblock \showarticletitle{Evolutionary Kernel Learning}.
\newblock In \bibinfo{booktitle}{\emph{Encyclopedia of Machine Learning and
  Data Mining}}. \bibinfo{publisher}{Springer}, \bibinfo{address}{New-York},
  \bibinfo{pages}{465--469}.
\newblock


\bibitem[\protect\citeauthoryear{Igel, Hansen, and Roth}{Igel
  et~al\mbox{.}}{2007}]%
        {Igel_2007}
\bibfield{author}{\bibinfo{person}{Christian Igel}, \bibinfo{person}{Nikolaus
  Hansen}, {and} \bibinfo{person}{Stefan Roth}.}
  \bibinfo{year}{2007}\natexlab{}.
\newblock \showarticletitle{Covariance Matrix Adaptation for Multi-objective
  Optimization}.
\newblock \bibinfo{journal}{\emph{Evol. Comput.}} \bibinfo{volume}{15},
  \bibinfo{number}{1} (\bibinfo{date}{March} \bibinfo{year}{2007}),
  \bibinfo{pages}{1--28}.
\newblock
\showISSN{1063-6560}


\bibitem[\protect\citeauthoryear{Igel, Heidrich{-}Meisner, and
  Glasmachers}{Igel et~al\mbox{.}}{2009}]%
        {Igel_2009a}
\bibfield{author}{\bibinfo{person}{Christian Igel}, \bibinfo{person}{Verena
  Heidrich{-}Meisner}, {and} \bibinfo{person}{Tobias Glasmachers}.}
  \bibinfo{year}{2009}\natexlab{}.
\newblock \showarticletitle{Shark}.
\newblock \bibinfo{journal}{\emph{Journal of Machine Learning Research}}
  \bibinfo{volume}{9} (\bibinfo{year}{2009}), \bibinfo{pages}{993--996}.
\newblock


\bibitem[\protect\citeauthoryear{Jordan}{Jordan}{2010}]%
        {Jordan_2010}
\bibfield{author}{\bibinfo{person}{M.~I. Jordan}.}
  \bibinfo{year}{2010}\natexlab{}.
\newblock \bibinfo{title}{Lecture notes: Justification for Bayes}.
\newblock
\newblock
\urldef\tempurl%
\url{https://people.eecs.berkeley.edu/~jordan/courses/260-spring10/lectures/lecture2.pdf}
\showURL{%
\tempurl}


\bibitem[\protect\citeauthoryear{{Loshchilov} and {Hutter}}{{Loshchilov} and
  {Hutter}}{2016}]%
        {Loshchilov_2016}
\bibfield{author}{\bibinfo{person}{Ilya {Loshchilov}} {and}
  \bibinfo{person}{Frank {Hutter}}.} \bibinfo{year}{2016}\natexlab{}.
\newblock \showarticletitle{{CMA-ES for Hyperparameter Optimization of Deep
  Neural Networks}}.
\newblock \bibinfo{journal}{\emph{arXiv e-prints}} \bibinfo{volume}{1604},
  \bibinfo{number}{April} (\bibinfo{date}{April} \bibinfo{year}{2016}),
  \bibinfo{pages}{arXiv:1604.07269}.
\newblock
\showeprint[arxiv]{cs.NE/1604.07269}


\bibitem[\protect\citeauthoryear{Marin and Robert}{Marin and Robert}{2007}]%
        {Marin_2007}
\bibfield{author}{\bibinfo{person}{Jean-Michel Marin} {and}
  \bibinfo{person}{Christian~P. Robert}.} \bibinfo{year}{2007}\natexlab{}.
\newblock \bibinfo{booktitle}{\emph{Bayesian Core: A Practical Approach to
  Computational Bayesian Statistics (Springer Texts in Statistics)}}.
\newblock \bibinfo{publisher}{Springer-Verlag}, \bibinfo{address}{Berlin,
  Heidelberg}.
\newblock
\showISBNx{0387389792}


\bibitem[\protect\citeauthoryear{Mühlenbein, Schomisch, and Born}{Mühlenbein
  et~al\mbox{.}}{1991}]%
        {Muhlenbein_1991}
\bibfield{author}{\bibinfo{person}{H. Mühlenbein}, \bibinfo{person}{M.
  Schomisch}, {and} \bibinfo{person}{J. Born}.}
  \bibinfo{year}{1991}\natexlab{}.
\newblock \showarticletitle{The parallel genetic algorithm as function
  optimizer}.
\newblock \bibinfo{journal}{\emph{Parallel Comput.}} \bibinfo{volume}{17},
  \bibinfo{number}{6} (\bibinfo{year}{1991}), \bibinfo{pages}{619 -- 632}.
\newblock
\showISSN{0167-8191}


\bibitem[\protect\citeauthoryear{Ollivier, Arnold, Auger, and Hansen}{Ollivier
  et~al\mbox{.}}{2017}]%
        {Ollivier_2017}
\bibfield{author}{\bibinfo{person}{Yann Ollivier}, \bibinfo{person}{Ludovic
  Arnold}, \bibinfo{person}{Anne Auger}, {and} \bibinfo{person}{Nikolaus
  Hansen}.} \bibinfo{year}{2017}\natexlab{}.
\newblock \showarticletitle{Information-geometric Optimization Algorithms: A
  Unifying Picture via Invariance Principles}.
\newblock \bibinfo{journal}{\emph{J. Mach. Learn. Res.}} \bibinfo{volume}{18},
  \bibinfo{number}{1} (\bibinfo{date}{Jan.} \bibinfo{year}{2017}),
  \bibinfo{pages}{564--628}.
\newblock
\showISSN{1532-4435}


\bibitem[\protect\citeauthoryear{Rastrigin}{Rastrigin}{1974}]%
        {Rastrigin_1974}
\bibfield{author}{\bibinfo{person}{L.~A. Rastrigin}.}
  \bibinfo{year}{1974}\natexlab{}.
\newblock \bibinfo{booktitle}{\emph{Systems of extremal control}}.
\newblock \bibinfo{publisher}{Mir}, \bibinfo{address}{Moscow}.
\newblock


\bibitem[\protect\citeauthoryear{Robert}{Robert}{2007}]%
        {Robert_2007}
\bibfield{author}{\bibinfo{person}{C.~P. Robert}.}
  \bibinfo{year}{2007}\natexlab{}.
\newblock \bibinfo{booktitle}{\emph{{The Bayesian Choice: From
  Decision-Theoretic Foundations to Computational Implementation}}}.
\newblock \bibinfo{publisher}{Springer}, \bibinfo{address}{New York}.
\newblock


\bibitem[\protect\citeauthoryear{Schervish}{Schervish}{1996}]%
        {Schervish_1996}
\bibfield{author}{\bibinfo{person}{M.J. Schervish}.}
  \bibinfo{year}{1996}\natexlab{}.
\newblock \bibinfo{booktitle}{\emph{Theory of Statistics}}.
\newblock \bibinfo{publisher}{Springer}, \bibinfo{address}{New York}.
\newblock
\showISBNx{9780387945460}
\showLCCN{97115378}
\urldef\tempurl%
\url{https://books.google.fr/books?id=F9A9af4It10C}
\showURL{%
\tempurl}


\bibitem[\protect\citeauthoryear{Varelas, Auger, Brockhoff, Hansen, ElHara,
  Semet, Kassab, and Barbaresco}{Varelas et~al\mbox{.}}{2018}]%
        {Hansen_2018}
\bibfield{author}{\bibinfo{person}{Konstantinos Varelas}, \bibinfo{person}{Anne
  Auger}, \bibinfo{person}{Dimo Brockhoff}, \bibinfo{person}{Nikolaus Hansen},
  \bibinfo{person}{Ouassim~Ait ElHara}, \bibinfo{person}{Yann Semet},
  \bibinfo{person}{Rami Kassab}, {and} \bibinfo{person}{Fr{\'{e}}d{\'{e}}ric
  Barbaresco}.} \bibinfo{year}{2018}\natexlab{}.
\newblock \showarticletitle{A Comparative Study of Large-Scale Variants of
  {CMA-ES}}.
\newblock \bibinfo{journal}{\emph{{PPSN} {XV} - 15th International Conference,
  Coimbra, Portugal}} \bibinfo{volume}{15}, \bibinfo{number}{2018}
  (\bibinfo{year}{2018}), \bibinfo{pages}{3--15}.
\newblock


\bibitem[\protect\citeauthoryear{Wikipedia}{Wikipedia}{2018}]%
        {wiki:CMAES}
\bibfield{author}{\bibinfo{person}{Wikipedia}.}
  \bibinfo{year}{2018}\natexlab{}.
\newblock \bibinfo{title}{CMA-ES}.
\newblock
\newblock
\urldef\tempurl%
\url{https://en.wikipedia.org/wiki/CMA-ES}
\showURL{%
\tempurl}


\end{thebibliography}

\end{document}